\documentclass{article}

\usepackage[preprint]{icml2026}
\usepackage{booktabs}
\usepackage{multirow}
\usepackage{tikz}
\usepackage{amsmath}
\usepackage[binary-units]{siunitx}
\usepackage[abbreviations]{foreign}
\usepackage{xspace}
\usepackage[inline]{enumitem}
\usepackage[algo2e,ruled,lined,boxed,commentsnumbered, noend, linesnumbered, procnumbered]{algorithm2e}

\usepackage[utf8]{inputenc} %
\usepackage[T1]{fontenc}    %
\usepackage{hyperref}       %
\usepackage{url}            %
\usepackage{booktabs}       %
\usepackage{amsfonts}       %
\usepackage{nicefrac}       %
\usepackage{microtype}      %
\usepackage{xcolor}         %
\usepackage{graphicx}
\usepackage{subcaption}
\usepackage{xspace}

\usepackage{pifont}%
\usepackage{ifthen}
\usepackage{acronym}
\acrodef{DL}{decentralized learning}
\acrodef{ML}{machine learning}
\acrodef{D-PSGD}{decentralized parallel stochastic gradient descent}
\acrodef{FL}{federated learning}
\acrodef{SGD}{stochastic gradient descent}
\acrodef{IID}{independent and identically distributed}
\acrodef{non-IID}{non independent and identically distributed}
\acrodef{RMSE}{root mean square error}
\acrodef{RMW}{random model walk}
\acrodef{GL}{gossip learning}
\acrodef{EL}{epidemic learning}
\acrodef{DWT}{discrete wavelet transform}
\acrodef{FFT}{fast Fourier transform}
\acrodef{MI}{mutual information}
\acrodef{DP}{differential privacy}
\acrodef{VN}{virtual node}
\acrodef{RN}{real node}
\acrodef{LDP}{local differential privacy}
\acrodef{PNDP}{pairwise network differential privacy}
\acrodef{PNLDP}{pairwise network local differential privacy}
\acrodef{GI}{gradient inversion}
\acrodef{CML}{collaborative machine learning}
\acrodef{TPR}{true positive rate}
\acrodef{FPR}{false positive rate}

\acrodef{LA}{linkability attack}
\acrodef{GIA}{gradient inversion attack}
\acrodef{MIA}{membership inference attack}
\acrodef{AIA}{attribute inference attack}
\acrodef{ROC}{receiver operating characteristic}
\acrodef{AUC}{area under the ROC curve} %
\newcommand{\sys}{Mosaic Learning\xspace}

\newcommand{\DL}{\ac{DL}\xspace}

\newcommand{\EL}{\ac{EL}\xspace}

\newcommand{\shakespeare}{\textsc{Shakespeare}\xspace}
\newcommand{\cifar}{\textsc{CIFAR-10}\xspace}
\newcommand{\cifarhundred}{\textsc{CIFAR-100}\xspace}

\newcommand{\movielens}{\textsc{MovieLens}\xspace}

\newcommand{\dpsgd}{{\xspace}\ac{D-PSGD}\xspace}

\newcommand{\niid}{\ac{non-IID}\xspace}

\graphicspath{ {figures/} }
\usepackage{tikz}
\usepackage[eulergreek]{sansmath}
\usepackage{pgfplots}
\usepackage{pgfplotstable}
\usepackage[capitalize,noabbrev]{cleveref}
\usepackage{comment}
\crefname{assumption}{assumption}{assumptions}

\pgfplotsset{compat=newest}
\usepgfplotslibrary{external,units,colorbrewer,groupplots,fillbetween,statistics}
\tikzsetexternalprefix{figures/}
\tikzset{external/mode=list and make}
\usetikzlibrary{patterns,shapes.misc}

\makeatletter

\newcommand{\newgroupwidth}[2]%
{\expandafter\xdef\csname groupwidth#1\endcsname{#2}}

\newcounter{groupwidth}
\newsavebox{\groupwidthbox}
\makeatletter
{\edef\groupnumber{#1}%
	\stepcounter{groupwidth}%
	\@ifundefined{groupwidth\thegroupwidth}{\pgfmathsetlengthmacro{\mywidth}{\linewidth/\groupnumber}}%
	{\expandafter\let\expandafter\mywidth\csname groupwidth\thegroupwidth\endcsname}%
	\begin{lrbox}{\groupwidthbox}%
		\tikzset{/pgfplots/width={\mywidth}}%
		\ignorespaces}%
	{\end{lrbox}%
	\usebox\groupwidthbox
	\pgfmathsetlengthmacro{\mywidth}{\mywidth + (\linewidth - \wd\groupwidthbox)/\groupnumber}
	\immediate\write\@auxout{\string\newgroupwidth{\thegroupwidth}{\mywidth}}}
\makeatother
\usepackage{amsmath,amssymb,amsfonts,amsthm}
\usepackage{physics} %
\usepackage{enumitem}
\usepackage{thm-restate}
\usepackage{bbm} %
\usepackage{stmaryrd} %
\usepackage{dsfont}

\theoremstyle{definition}

\theoremstyle{remark}
\newtheorem{remark}{Remark}
\newtheorem{theorem}{Theorem}
\newtheorem{lemma}[theorem]{Lemma}

\newtheorem{assumption}{Assumption}

\allowdisplaybreaks

\newcommand{\R}{\mathbb{R}}

\newcommand{\espsmallcond}[2]{\mathbb{E}_{#1}\left[#2\right]}
\newcommand{\esp}[1]{\espsmallcond{}{#1}}

\newcommand{\transpose}[1]{#1^{\top}}
\newcommand{\half}{\nicefrac{1}{2}}

\newcommand{\nbparams}{d}

\newcommand{\identity}[1]{I_{#1}}

\newcommand{\modelspace}[0]{\R^{\nbparams}}

\newcommand{\lr}[1]{\eta_{}}
\newcommand{\globalloss}[0]{F}
\newcommand{\globalgradient}[0]{\nabla{}\globalloss{}}
\newcommand{\localloss}[1]{F_{#1}}
\newcommand{\localgradient}[1]{\nabla{} \localloss{#1}}
\newcommand{\smoothnessconstant}[0]{L}

\newcommand{\integerinterval}[1]{\left\llbracket#1\right\rrbracket}

\newcommand{\model}[2]{x_{#1}^{#2}}

\newcommand{\datapoint}[2]{\xi_{#1}^{#2}}
\newcommand{\boundstochasticnoise}[0]{\sigma}
\newcommand{\boundheterogeneity}[0]{\mathcal{H}}

\newcommand{\kronecker}[0]{\otimes}
\newcommand{\commutingmatrix}[2]{\mathbf{K}^{(#1,#2)}}
\newcommand{\chunkoperator}[0]{\mathcal{C}}
\newcommand{\chunk}[1]{\chunkoperator\left(#1\right)}
\newcommand{\chunkproj}[1]{\Pi^{(#1)}}

\newcommand{\one}[0]{\mathds{1}}

\newcommand{\E}[0]{\mathbb{E}}
\makeatletter
\AtBeginDocument{%
	\def\ltx@label#1{\cref@label{#1}}%

	\def\label@in@display@noarg#1{\cref@old@label@in@display{#1}}%

	\def\label@in@mmeasure@noarg#1{%
		\begingroup
		\measuring@false
		\cref@old@label@in@display{#1}%
		\endgroup
	}%
}
\makeatother

\begin{document}

	\twocolumn[
	\icmltitle{Mosaic Learning: A Framework for Decentralized Learning \\with Model Fragmentation}
	\icmltitlerunning{Mosaic Learning: A Framework for Decentralized Learning with Model Fragmentation}

	\icmlsetsymbol{equal}{*}

	\begin{icmlauthorlist}
		\icmlauthor{Sayan Biswas}{epfl}
		\icmlauthor{Davide Frey}{inria}
		\icmlauthor{Romaric Gaudel}{inria}
		\icmlauthor{Nirupam Gupta}{univcopenhagen}
		\icmlauthor{Anne-Marie Kermarrec}{epfl}
		\icmlauthor{Dimitri Lerévérend}{equal,inria}
		\icmlauthor{Rafael Pires}{epfl}
		\icmlauthor{Rishi Sharma}{equal,epfl}
		\icmlauthor{François Taïani}{inria}
		\icmlauthor{Martijn de Vos}{epfl}
	\end{icmlauthorlist}

	\icmlaffiliation{epfl}{École Polytechnique Fédérale de Lausanne (EPFL), Switzerland}
	\icmlaffiliation{inria}{Univ Rennes, Inria, CNRS, IRISA, France}
	\icmlaffiliation{univcopenhagen}{University of Copenhagen, Denmark}

	\icmlcorrespondingauthor{Dimitri Lerévérend}{dimitri.lereverend@inria.fr}
	\icmlcorrespondingauthor{Rishi Sharma}{rishi.sharma@epfl.ch}

	\icmlkeywords{Machine Learning, ICML}

	\vskip 0.3in
	]

	\printAffiliationsAndNotice{}

\begin{abstract}

\Ac{DL} enables collaborative \ac{ML} without a central server, making it suitable for settings where training data cannot be centrally hosted.
We introduce \sys, a \DL framework that decomposes models into fragments and disseminates them independently across the network.
Fragmentation reduces redundant communication across correlated parameters and enables more diverse information propagation without increasing communication cost.
We theoretically show that \sys \begin{enumerate*}[label=\textit{(\roman*)}]
\item shows state-of-the-art worst-case convergence rate, and
\item leverages parameter correlation in an \ac{ML} model, improving contraction by reducing the highest eigenvalue of a simplified system.
\end{enumerate*}
We empirically evaluate \sys on four learning tasks and observe up to $12$ percentage points higher node-level test accuracy compared to \ac{EL}, a state-of-the-art baseline.
In summary, \sys improves \ac{DL} performance without sacrificing its utility or efficiency, and positions itself as a new \ac{DL} standard.

\end{abstract}

\section{Introduction}
\label{sec:intro}

\Acf{DL} allows nodes to collaboratively train a \acf{ML} model across multiple nodes while keeping their data local and without the need for a central server orchestrating the learning process~\cite{lian2017can}.
Typically, in each round of \ac{DL}, nodes first perform local training on their private datasets.
The resulting model updates, in their entirety, are then exchanged with neighboring nodes according to the communication topology and locally aggregated.
The aggregated model serves as the initialization for the subsequent round, and this process is repeated until convergence.
As \ac{DL} does not rely on a central coordinating entity (\eg a server), it is a scalable and robust approach that does not require trusting a single entity. %

Since the introduction of \dpsgd, the canonical algorithm for \ac{DL}, by Lian et al. in 2017~\cite{lian2017can}, many variants have been introduced to overcome various challenges that are innate to \dpsgd, such as learning under \niid data distributions~\cite{lian2017can,koloskova2020decentralized,devos2023epidemic}.
Among these, model fragmentation, the idea of splitting a model into smaller pieces that are sent to or collected by other nodes independently, has emerged as a particularly promising design pattern.
Prior approaches like \textsc{Shatter}~\cite{biswas2024noiseless}, \textsc{Yoga}~\cite{liu2023yoga} and \textsc{DivShare}~\cite{biswas2025boosting} demonstrate that fragmenting model updates can provide additional privacy protection by limiting the exposure of full models, mitigate asynchrony by reducing straggler effects, reduce communication cost, and improve convergence speed.
However, in these works, fragmentation is primarily treated as a means to address isolated system concerns, rather than as a first-class learning primitive.
As a result, a holistic understanding of the benefits, trade-offs, and learning dynamics induced by model fragmentation in \ac{DL} is still missing. In particular, how fragmenting 
models exactly affect the learning process in decentralized frameworks remains largely unexplored.

This work introduces \sys as a unified class of \ac{DL} algorithms with model fragmentation and, hence, develops a comprehensive understanding of how model fragmentation can enhance the learning process of \ac{DL}.
In summary, our contributions are as follows:
\begin{itemize}
    \item We introduce \sys, a novel unified \ac{DL} framework that partitions local models into discrete segments for independent distribution, facilitating more rapid and diverse information dissemination across the network (\Cref{sec:design}).
    \item  We establish theoretical convergence guarantees for \sys, proving that it matches the convergence rate of \ac{EL}~\cite{devos2023epidemic}, the current state-of-the-art \ac{DL} baseline (\Cref{subsec:theory convergence classical}). 
    \item We demonstrate that \sys accelerates learning in convex landscapes by mitigating information redundancy and fostering the sharing of uncorrelated model parameters in the network (\Cref{subsec:theory impact of chunks}). This analysis provides fundamental insights into the structural benefits that fragmentation-based approaches bring to decentralized optimization.
    \item Experimentally, we demonstrate across four learning tasks that \sys outperforms \ac{EL} by up to 12 percentage points in settings with high label heterogeneity, while maintaining parity with \ac{EL} in homogeneous (IID) settings (\Cref{sec:evaluation}).
\end{itemize}

\section{Background and Preliminaries}
\label{sec:prelims}

This works introduces a new \acf{DL} framework~\cite{lian2017can}.
We first describe \dpsgd, the standard approach to perform \ac{DL}.
Our work builds on \ac{EL}, a variant that improves the convergence rate of \dpsgd, and we describe the \ac{EL} workflow afterwards~\cite{devos2023epidemic}.

\subsection{Standard decentralized learning (\dpsgd)}
In \dpsgd, there is a system composed of $n$ nodes collaboratively training an \ac{ML} model of $d$ parameters.
Each node $i$ has access to a local dataset $\mathcal{D}_i$, owns a local model $x_t^{(i)}$ in iteration $t$, and aims to minimize a local objective function $f_i:\modelspace \times \mathcal{Z} \to\R_{\geq 0}$ defined over the model space $\R^d$ and a data space $\mathcal{Z}$.
The local datasets never leave the corresponding nodes' devices.
The learning objective is to minimize the average loss over all nodes; formally,
\begin{align}\label{eq:global objective}
    \min_{x\in\modelspace} F(x) = \frac{1}{n}\sum_{i=1}^{n} F_i(x)\nonumber
\end{align}
where $F_i(x) = \E_{\xi_i \sim \mathcal{D}_i} \left[f_i(x,\xi_i)\right]$ is the expected local loss at node $i$ over a minibatch $\xi_i$ sampled randomly from $\mathcal{D}_i$.

A training session with \dpsgd proceeds in $T$ iterations and let the local model held by node $i$ in any iteration $t$ $0\leq t\leq T$ be denoted by $x_t^{(i)}$. Each node $i$ starts by initializing a random $x_0^{(i)}\in \mathbb{R}^d$.
For iterations $ 1 \leq t \leq T-1 $, each node $ i $ trains its local model $x_t^{(i)}$ using \ac{SGD}, producing $x_{t+1/2}^{(i)}$.
Nodes communicate along a \emph{communication graph} in each iteration $t$, which is modeled by a communication matrix $W_t$.
In \dpsgd, $W_t=W_{t'}$ for every $0\leq t,t'\leq T-1$, \ie the communication matrix is the same across every iteration.
\Cref{fig:mosaic_learning} (left) shows the sending and receiving of models in \dpsgd.
After local training, $i$ sends this trained model to its neighbors given by $ W_t $, and receives trained models from all its neighbors in $ W_t $ (step 1, left).
Node $ i $ then integrates the received model parameters from its neighbors with its own using a weighted aggregation scheme (step 2, left).
It is typical to aggregate all the models with equal weights and yhe aggregated model is used as the local model in the next iteration $ t + 1 $. The process continues until convergence. 

\subsection{\Acf{EL}}
\Ac{EL} is a \ac{DL} algorithm that aims to spread information using an epidemic protocol~\cite{devos2023epidemic}.
Instead of keeping the gossip matrix $ W $ fixed across iterations, \ac{EL} randomizes $ W $ between rounds.
Thus, at each iteration, nodes asymmetrically send their model to a set of random neighbors.
This is captured by a communication matrix $W_t$ that is \emph{row stochastic} but not necessarily \emph{column stochastic}, \ie, $W_t \one_n = \one_n$, but $\transpose{\one_n} W_t \neq \one_n$.
\ac{EL} exhibits faster convergence compared to \dpsgd and related topology construction algorithms, which is shown both theoretically and empirically, making it a state-of-the-art \ac{DL} approach.

\begin{figure*}[t]
    \centering
    \includegraphics[width=\linewidth]{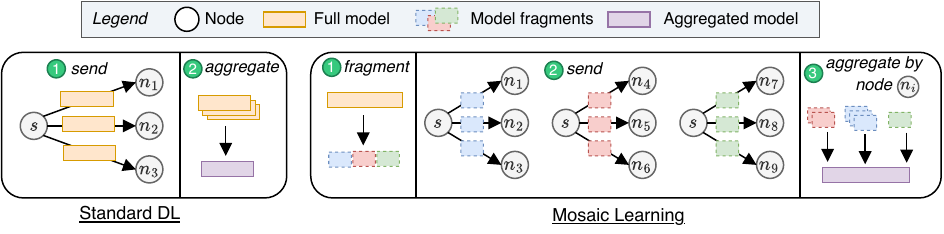}
    \caption{Sending and receiving models in standard \acf{DL} (left) and \sys (right).}
    \label{fig:mosaic_learning}
\end{figure*}

\begin{algorithm}[t]
    \caption{\sys from the view of node $i$}%
    \label{alg:system}
    \begin{algorithmic}[1]
        \STATE \textbf{Input:} no. of fragments $K$, stepsize $\lr{}$, iterations $T$, no. of local \ac{SGD} steps $H$
        \STATE \textbf{Init:} $x_0^{(i)}$ %
        \FOR{$t\in \integerinterval{0, T-1}$}
         \STATE \textbf{Init:} gossip matrices $\{W_t^{(k)}\}_{t,k}$ for $k=1,\ldots,K$
            \STATE $\tilde{x}_t^{(i,0)}\leftarrow x_t^{(i)}$
            \FOR{$h=0,\ldots,H-1$}
            \STATE Draw $\datapoint{t}{(i)}\sim \mathcal{D}_i$
            \STATE $\tilde{g}_t^{(i,h)} \leftarrow \nabla f_i(\tilde{x}_t^{(i,h)}, \datapoint{t}{(i)})$
            \STATE $\tilde{x}_{t}^{(i,h+1)}\leftarrow \tilde{x}_t^{(i,h)}-\lr{}\,\tilde{g}_t^{(i,h)}$
            \ENDFOR
            \STATE $x_{t+\half}^{(i)}\leftarrow \tilde{x}_{t+\half}^{(i,H)}$
            \FOR{$k=1,\dots,K$}
                \STATE \textbf{Share} fragment $k$ of $x_{t+\half}^{(i)}$ using $W_t^{(k)}$
                \STATE \textbf{Receive} fragment $k$ from other nodes 
            \ENDFOR
            \STATE $x_{t+1}^{(i)} \leftarrow$ (fragment-wise) avg. of received models
        \ENDFOR
        \STATE \textbf{Output:} $\{x_T^{(i)}\}_{i=1}^n$
    \end{algorithmic}
\end{algorithm}

\section{The \sys framework}
\label{sec:design}

Here, we present \sys, a unified \ac{DL} framework that leverages model fragmentation to facilitate more rapid and diverse information dissemination in the network.
We focus on model fragmentation as a first-class design choice because it exposes a fundamental degree of freedom in \ac{DL} algorithms that is largely orthogonal to existing algorithmic improvements and remains theoretically unexplored.
Prior works have shown that splitting models into fragments can be beneficial for enhancing privacy and handling asynchrony, however, these benefits have been studied in isolation and from a systems perspective~\cite{biswas2024noiseless,biswas2025boosting}. %
The main idea is that fragmentation directly affects how information propagates and mixes across the network during training.

To analyze the effects of model fragmentation, we present the \sys algorithm in \Cref{alg:system} (from the perspective of node $i$) and visualize the sending and receiving of model fragments in Mosaic Learning in \Cref{fig:mosaic_learning} (right).
On a high level, it conserves the usual structure of \ac{DL} algorithms where nodes continuously train their local models and share them with neighbors.
The algorithm starts by each node $ i $ initializing its local model $ x_0^{(i)} $ (line 2).
At the start of each iteration $ t $, we initialize gossip matrices $\{W_t^{(k)}\}_{t,k}$ for $k=1,\ldots,K$ (line 4).
For our theoretical analysis (\Cref{sec:analysis}), we make the same assumption as \ac{EL}, namely that $ W_t $ is row stochastic but not necessarily column stochastic.
A node then trains its current local model $x_t^{(i)}$ for $ H $ local SGD steps (line 6-10), resulting in the trained model $\tilde{x}_{t+\half}^{(i,H)}$ (line 11).

\textbf{Model Fragmentation.}
In contrast to \dpsgd and \ac{EL}, nodes in \sys fragment their models into $ K $ chunks using a mapping $\chunkoperator$ and send each chunk to $r$ neighbors.
Formally, we view fragmentation as a mapping $\chunkoperator: \integerinterval{1,d} \to \integerinterval{1,K}$ that assigns each parameter coordinate to a fragment.
Equivalently and more compactly, we work with orthogonal projectors $\chunkproj{k}:\R^d\to\R^d$ for $k=1,\dots,K$ that selects a subset of parameters from the vector $x$ associated to fragment $k$.
Formally, we have $(\chunkproj{k}x)[i] = (x)[i] \iff \chunkoperator(i) = k$, and 0 otherwise).
We consider \emph{disjoint} fragments, defined by:
\[
\chunkproj{k}\chunkproj{q}=0\ (k\ne q),\qquad \sum_{k=1}^K \chunkproj{k} = I_d.
\]

For simplicity, we assume each fragment has the same number of parameters ($\operatorname{tr}\left(\chunkproj{k}\right) = d/K$).
Furthermore, the fragmentation algorithm is fixed across iterations.
The parameters in fragment $k$ are shared along a \emph{distinct} communication matrix $W_t^{(k)}$, which enables better spread of information.
The model dissemination process in \sys is also shown in \Cref{fig:mosaic_learning} (right, step 2), in a setting with $ K=3$ and $r=3$, where a node $ s $ sends fragments to neighbors $ n_1 $ to $ n_9 $ in $ W_t^{(k)}$.
Unlike \dpsgd where $ s $ sends the full model to $ r=3 $ neighbors, with \sys $ s $ sends each of its $k=3$ model fragments to 3 neighbors (line 13 in \Cref{alg:system}).

\textbf{Fragment-wise Aggregation.}
With \sys, a node $ i $ within a single iteration $ t $ receives fragments from neighboring nodes (line 14).
Depending on the gossip matrix $ W_t^{(k)} $, nodes may receive a different number of fragments at different fragment indices. %
This is visualized in \Cref{fig:mosaic_learning} (right, step 3) where a node $ i $ receives two fragments at index 1, three fragments at index 2, and one fragment at index 3.
Node $ i $ then performs fragment-wise average:
\begin{align}\label{eq:update rule chunks for algo}
    \chunkproj{k} x_{t+1}^{(i)}:=  \sum_j W_t^{(k)}[i,j]\,\chunkproj{k} x_{t+\half}^{(j)}.
\end{align}

\textbf{Discussion.}
\sys induces the same communication footprint of classical \ac{DL}: although model updates are partitioned into fragments and disseminated independently, each node communicates the same total number of parameters per iteration as in \dpsgd or \ac{EL}.
\sys therefore alters the structure of information flow without increasing communication cost or introducing additional synchronization requirements.
Importantly, \sys does not rely on asynchrony, sparsification, or partial participation; it modifies only the gossip step while leaving the local optimization procedure unchanged.

\begin{remark}\label{rem:EL instance of MoL}
\ac{EL} corresponds to \sys with $ K = 1 $, sharing the entire model on a single communication graph.
\end{remark}
\section{Theoretical Analysis}
\label{sec:analysis}
In this section, we present the formal analysis to capture the benefits of model fragmentation in the learning aspects of \ac{DL}.
First, we show that \sys is guaranteed to converge as fast as existing literature, %
regardless of the number of fragments $K$. Further, we show that a higher $K$ results in faster consensus in simpler convex landscapes.
All results in this section hold independently of the fragmentation  $\chunkoperator$; \ie the specific heuristic used to generate the $K$ model fragments does not affect the convergence guarantees of \sys.

Our theoretical contributions are twofold:
\begin{enumerate*}[label=\textit{(\roman*)}]
\item We derive state-of-the-art worst-case convergence rates under multiple communication settings (\cref{subsec:theory convergence classical}).
\item Using a simplified formalism, we show that increasing the number of fragments improves consensus in the convex case (\cref{subsec:theory impact of chunks}).
\end{enumerate*}

We impose no specific constraints on the communication matrix, except for the general convergence analysis (\cref{subsec:theory convergence classical}), which requires $W_t$ to be undirected and regular. This aligns with standard assumptions in the decentralized optimization literature~\cite{devos2023epidemic,pmlr-v119-koloskova20a}. Notably, our analysis of \sys in the convex landscape (\cref{subsec:theory impact of chunks}) remains agnostic to the underlying network topology. We emphasize that all the results presented below hold regardless of whether the topology is dynamic ($W_t$ varies each round) or static ($W_t = W, \forall t$).

\subsection{General convergence}
\label{subsec:theory convergence classical}

First, we show the worst-case convergence rate of \sys matches prior work.
Intuitively, the number of fragments $K$ does not affect the average behavior of gossip: although we have $K$ different (independent) gossip matrices, per-fragment matrices share the same distribution and their aggregated expectation equals the single-fragment operator. Thus, the usual convergence analysis applies \emph{per-fragment}.

We proceed by considering the standard assumptions in distributed optimization~\cite{lian2017can,pmlr-v119-koloskova20a,devos2023epidemic}.
For completeness, we refer the reader to~\cref{sec:assumptions} in the Appendix for a detailed description of these assumptions.
Under these assumptions, we derive state-of-the-art convergence guarantees for \sys with the following theorem.
\begin{restatable}[Convergence of \sys]{theorem}{convergenceEpidemic}\label{thm:convergence of sys}
    Consider the typical assumptions of \begin{enumerate*}[label=\textit{(\roman*)}]
        \item smoothness,
        \item bounded stochastic noise, and
        \item bounded heterogeneity
    \end{enumerate*}. Consider \sys as described in~\cref{alg:system} with $K$ fragments.
    Then, with an appropriate choice of stepsize $\lr{}$ and after $T$ iterations, it holds that $\frac{1}{nT} \sum_{i=1}^{n}\sum_{t=0}^{T-1}\E\left[\norm{\nabla F\left(x_t^{(i)}\right)}^2\right] \leq \epsilon$ as long as $\epsilon$ and $\lr{}$ are bounded by the same rate as in \ac{EL}~\cite{devos2023epidemic}, independently of the number of fragments $K$.
\end{restatable}
\Cref{thm:convergence of sys} shows that in the worst-case scenario, fragmentation is at least as good as existing methods.

\begin{remark}
    The worst-case convergence rate presented in~\cref{thm:convergence of sys} is independent of the number of fragments $K$.
\end{remark}

\begin{remark}
    Under additional assumptions of a doubly-stochastic mixing matrix, a tighter convergence rate can be derived by following the same line of analysis as given by Koloskova \etal in Theorem 2 of their convergence analysis in standard \ac{DL}~\cite{pmlr-v119-koloskova20a}.
\end{remark}

\subsection{Impact of the number of fragments}%
\label{subsec:theory impact of chunks}

The worst-case convergence derived by in~\cref{thm:convergence of sys} does not capture how fragmentation itself reduces redundant communication across correlated parameters or how it changes the contraction properties of the consensus step. To make this effect explicit, we analyze a simplified quadratic setting where the correlation between model parameters is captured by a positive-definite matrix $A$:
\begin{assumption}
    The local losses are quadratic and identical across nodes: $\forall i,\; f_i(x)=f(x)=\|x-x^*\|_A^2$ with $A\succ0$, $A= \transpose{A}$.
\end{assumption}

We concatenate local models into a vectorized form $X_t=\begin{pmatrix}x_t^{(1)} & \cdots & x_t^{(n)}\end{pmatrix}^\top$ and write $X^*=\mathbf{1}_n\otimes x^*$. The gradient step is
\begin{align}\label{eq:vectorized gradient update}
    X_{t+\half}=X_t-2\lr{}\,(\identity{n}\kronecker A)\,(X_t-X^*).
\end{align}

These notations allow us to capture the effect of fragmentation on gossip without tensor operations. Consider $K$ gossip matrices $W_t^{(k)}$, each associated to a chunk. We define:
\begin{align}\label{eq:def gossip matrix under chunking}
    \mathbf{W}_t
        :=\mathrm{diag}(W_t^{(\chunk{1})},\dots,W_t^{(\chunk{d})})=\sum_{k=1}^K \chunkproj{k}\kronecker W_t^{(k)}.
\end{align}
This matrix is block-diagonal, with each block associated with the gossip matrix for a parameter.
For $K=1$ (no fragmentation), $\mathbf{W}_t = \identity{d} \kronecker W_t^{(1)}$. The number of fragments $K$ controls the number of distinct blocks in $\mathbf{W}_t$.

To match the node-wise ordering of $X_t$ we use the \emph{commuting} (or vec-permutation) matrix $\commutingmatrix{n}{d}$~\cite{loanUbiquitousKroneckerProduct2000, hendersonVecpermutationMatrixVec1981, langvilleKroneckerProductStochastic2004}.
$\commutingmatrix{n}{d}$ permutes Kronecker-product orderings so that blocks correspond to nodes rather than parameter groups; this makes the block-diagonal structure in~\cref{eq:def gossip matrix under chunking} compatible with the stacked vector $X_t$.
The gossip step thus becomes:
\begin{align}\label{eq:vectorized gossip update}
    X_{t+1} = \commutingmatrix{n}{d} \mathbf{W}_t \commutingmatrix{d}{n}X_{t+\half}.
\end{align}

Combining~\cref{eq:vectorized gossip update,eq:vectorized gradient update}, we can view the evolution of $X_t$ as a linear system.
This allows us to analyze the impact of fragmentation on consensus:
\begin{restatable}[Consensus error evolution]{lemma}{vectorizedconvergence}\label{thm:consensus error evolution}
    Consider $e_t= X_t - \bar{X}_t$ to be the consensus error at iteration $t$. Then, the consensus error evolves according to:
    \begin{align*}
        e_{t+1} = P \commutingmatrix{n}{d} \mathbf{W}_t \commutingmatrix{d}{n} \left( \identity{n} \kronecker \left( \identity{d} - 2 \lr{} A \right) \right) e_t,
    \end{align*}
    where $P = \left(\identity{n} - \tfrac{1}{n}\one_n \one_n^\top\right) \kronecker I_d$ is the projector onto the disagreement subspace.
\end{restatable}
Using~\cref{thm:consensus error evolution}, we show that the consensus distance is driven by the matrix:
\begin{align*}
    M_t := P\commutingmatrix{n}{d} \mathbf{W}_t \commutingmatrix{d}{n} \left( \identity{n} \kronecker \left( \identity{d} - 2 \lr{} A \right) \right).
\end{align*}
In particular, the norm of the consensus distance is governed by $\rho(\transpose{M_t}M_t)$, the largest eigenvalue of $\transpose{M_t}M_t$,  as we consider the norm of the consensus error $\transpose{e_t}e_t$.
We now consider two examples, with two types of correlation between model parameters (represented by matrix $A$). Our numerical simulations using $2$-regular gossip matrices, in~\cref{fig:impact_of_K_on_consensus}, show that $\rho(\transpose{M_t}M_t)$ decreases when the number of fragments $K$ increases. Hence, the consensus distance also decreases faster, as confirmed by the results in~\cref{fig:impact_of_K_on_consensus_error}.

\begin{figure}[t]
    \centering
    \includegraphics{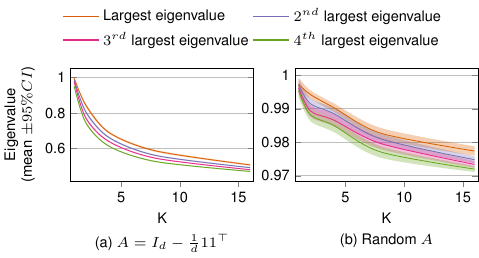}
    \caption{Eigenvalues of $\transpose{M}M$ as a function of $K$ for two examples with $n=50$ nodes and $d=16$ parameters. Increasing $K$ decreases the contraction factor, improving consensus.}%
    \label{fig:impact_of_K_on_consensus}
\end{figure}

\begin{figure}[t]
    \centering
    \includegraphics{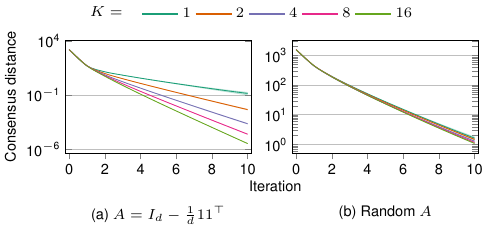}
    \caption{Consensus distance $\norm{X_t - \bar{X}_t}^2$ as a function of the number of fragments $K$ for two examples with $n=50$ nodes and $d=16$ parameters. Increasing $K$ improves consensus speed.}%
    \label{fig:impact_of_K_on_consensus_error}
\end{figure}

\section{Evaluation}
\label{sec:evaluation}

We now evaluate the performance of \sys.
Our evaluation answers the following questions.
\textbf{RQ1:}~How does \sys perform across datasets as we vary the number of fragments (\Cref{sec:eval_fragments})?
\textbf{RQ2:}~What is the effect of \sys on the individual node models and consensus distance (\Cref{sec:eval_fragments})?
\textbf{RQ3:}~What is the effect of increasing graph degree on \sys (\Cref{sec:eval_topology})?
\textbf{RQ4:}~What is the effect of data heterogeneity on \sys (\Cref{subsec:eval niidness})?

\subsection{Experimental Setup}

\textbf{Implementation and infrastructure.}
We implement \sys using the codebase of \textsc{Shatter}~\cite{biswas2024noiseless}\footnote{Code will be made publicly available upon acceptance.}.
The experimental environment consists of 3 machines, each with a dual Intel Xeon E5-2630 v3 processor running at 2.40GHz with 8 cores per processor.
The systems have hyperthreading enabled and operate Ubuntu 22.04.5 LTS using the 5.15.0-164-generic kernel version.

\textbf{Datasets, models, hyperparameters, and baseline.}
We use non-IID train and test datasets from \cifar, \cifarhundred~\cite{krizhevsky2014cifar} (in the Appendix), \movielens-small~\cite{grouplens:2021:movielens}, and a subsampled version of the \shakespeare dataset from the LEAF benchmark~\cite{leaf} (in the Appendix).
For \cifar and \cifarhundred, we use GN-LeNet model~\cite{hsiehskewscout2020}.
For \movielens, we employ collaborative filtering through matrix factorization~\cite{korenmatrixfactorization2009}.
Finally, for \shakespeare's next character prediction task, we use a stacked LSTM~\cite{leaf}.
The learning rate is determined through a grid search using a validation set.
We use \EL as the baseline \DL algorithm, which corresponds to $K=1$ (see \cref{rem:EL instance of MoL}).

\textbf{Metrics.}
We evaluate performance using four metrics.
First, \emph{node-average} performance measures the mean performance across all individual node models evaluated on a global test set.
This captures how well individual participants perform after training.
Second, \emph{average-model} performance measures the performance of a single model obtained by averaging all model parameters of all nodes.
This reflects the quality of a global model if we were to aggregate all models together~\cite{zhuSurprisingEffectivenessSingle2025}.
For classification tasks (\cifar, \cifarhundred, \shakespeare), we report accuracy; for the recommendation task (\movielens), we report RMSE loss.
Third, we report the \emph{consensus distance} computed as the mean of the $\ell^2$ norm between each node's parameters and the network-wide average model~\cite{kongConsensusControlDecentralized2021}.
Finally, the \emph{standard deviation} of node performance measures the heterogeneity in individual node performances, capturing fairness and consistency of learning outcomes across participants.

\subsection{The effect of the number of fragments}
\label{sec:eval_fragments}

\begin{figure}[t]
    \centering
    \includegraphics{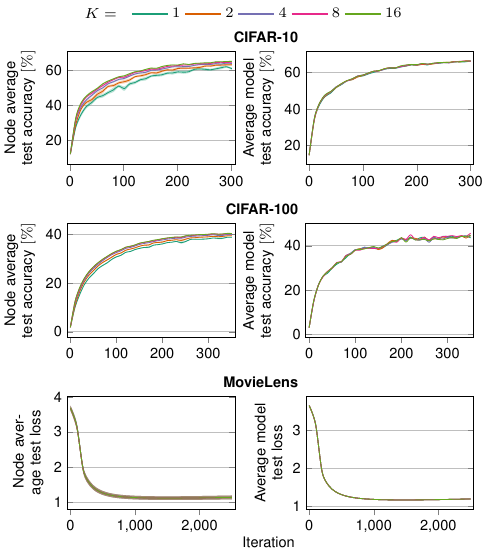}
    \caption{Performance of \sys across iterations and number of fragments ($K$) for \cifar (top), \cifarhundred (middle), and \movielens (bottom), showing node-average (left) and model-average (right) test accuracies. }
    \label{fig:plot1_performance}
\end{figure}

\Cref{fig:plot1_performance}  shows the impact of model fragmentation on the performance of \sys using a regular-graph communication network with a node degree of 8, using \cifar (top), \cifarhundred (middle), and \movielens (bottom), for different numbers of fragments $K$.
The left column show the average test accuracy and loss across nodes, whereas the right column shows the test accuracy of the averaged model across all nodes.
The top-left plots show that node accuracy increases with $K$ for \cifar and \cifarhundred, highlighting the positive impact of model fragmentation.
The top-right plots, on the other hand, shows that the average model accuracy remains the same regardless of $K$, confirming that model fragmentation does not negatively affect average model convergence.

Results on \movielens (bottom plots) are more nuanced, and both node-average and average-model loss values appear unaffected by the number of fragments $ K $. \Cref{fig:plot1a} in the Appendix shows similar results for \shakespeare.

\begin{figure}[t]
    \centering
    \includegraphics{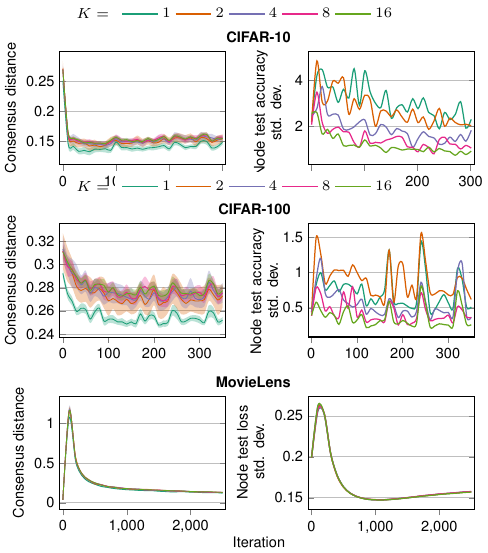}
    \caption{Consensus distance (left) and standard deviation of node performance (right) across iterations and $K$, for \cifar (top), \cifarhundred (middle), and \movielens (bottom), on a network with degree 8.}
    \label{fig:plot2_consensus}
\end{figure}

To gain better insight into the behavior of \sys, \Cref{fig:plot2_consensus} complements the aforementioned results by showing the consensus distance and the standard deviation of node test accuracy on the three datasets for various values of $K$.
We first discuss the results for \cifar and \cifarhundred (top two rows).
The overall consensus distance across iterations appears to increase slightly with $K$, which is different from the convex case analyzed in \Cref{subsec:theory impact of chunks}.
On the other hand, the standard deviation of node test accuracy clearly decreases with $K$.
This explains the better performance of large values of $K$ with respect to node-average test accuracy, and essentially shows that fragmentation provides some form of ``parameter mixing'', despite the increase in consensus distance. This suggests that consensus distance does not constitute a reliable metric to reflect performance in non-convex settings.
The bottom plots show, instead, that on \movielens, both consensus distance and the standard deviation of node accuracy remain unaffected by $K$, confirming the fact that fragmentation has no or little impact on this learning task. On \shakespeare (\Cref{fig:plot2a} in the Appendix), fragmentation increases consensus distance but standard deviation shows only minimal variations.%

\subsection{The effect of graph degree}
\label{sec:eval_topology}

\begin{figure}[t]
    \centering
    \includegraphics{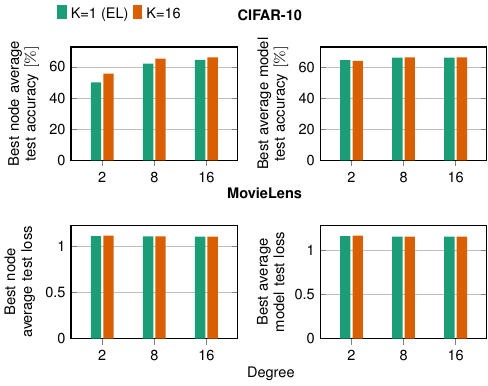}
    \caption{Best node-average (left) and average-model (right) performance across network degrees and values of $K$, for \cifar (top) and \movielens (bottom).}
    \label{fig:plot3_degree_best}
\end{figure}

Next, we examine the impact of the topology of the communication network by varying the degree of the regular communication graph.
\Cref{fig:plot3_degree_best} shows two configurations, $K=1$ and $K=16$, on three different graphs, with degrees of $2$, $8$, and $16$. The two top plots show node-average and average-model test accuracies in \cifar, while the bottom plots show node-average and average-model loss values in \movielens.
In \cifar, denser graphs result in faster mixing, leading to better convergence in terms of node-average test accuracy in both non-fragmented and fragmented configurations. The same consideration, albeit more nuanced, applies to average-model test accuracy, which, however, already reaches close-to-top values even with a sparse regular graph of degree $2$.
In \movielens, results follow the trend highlighted in the previous section: both performance metrics appear unaffected by node degree.

\begin{figure}[t]
    \centering
    \includegraphics{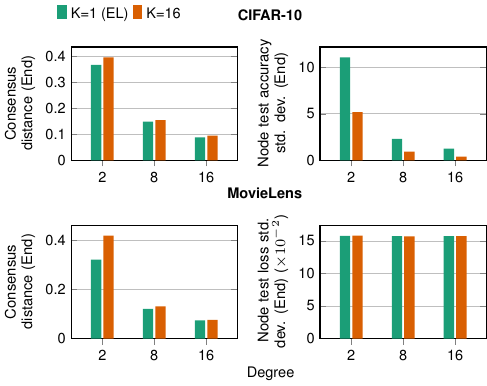}
    \caption{Consensus distance (left) and standard deviation of node performance (right) across network degrees and values of $K$, for \cifar (top) and \movielens (bottom).}
    \label{fig:plot4_degree_end}
\end{figure}

\Cref{fig:plot4_degree_end} complements these results by showing the final consensus distance and standard deviation of node accuracy/loss with regular-graph degrees of $2$, $8$, and $16$. In sparse graphs, fragmentation leads to a significant increase in consensus distance, even in the \movielens case, even though performance (\Cref{fig:plot3_degree_best}) does not decrease, and rather increases in \cifar. Moreover, a denser graph results, as expected, in lower consensus distances.

The standard deviation plots confirm the fact that this metric effectively explains the differences in node-average test accuracy. In \cifar, the standard deviation of node test accuracy decreases both with fragmentation (as already observed) and with the degree of the graph. Lower standard deviation, therefore, correlates with better node-average test accuracy.
In \movielens, on the other hand, the standard deviation appears independent of the communication graph's degree, which is consistent with the fact that accuracy is unaffected by graph density.
\subsection{The effect of data heterogeneity}\label{subsec:eval niidness}

\begin{figure}[h!]
    \centering
    \includegraphics{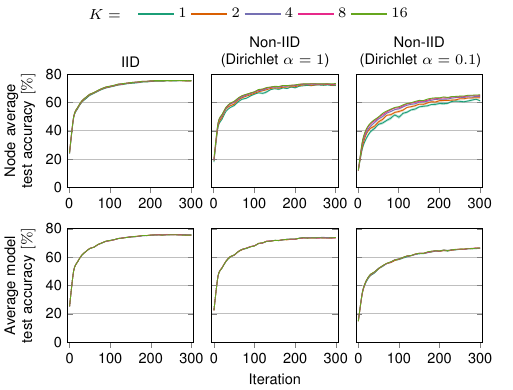}
    \caption{Node-average (top) and average-model (bottom) performance for \cifar across iterations, data distributions and values of $K$.}
    \label{fig:plot5_distribution_perf}
\end{figure}

\begin{figure}[h!]
    \centering
    \includegraphics{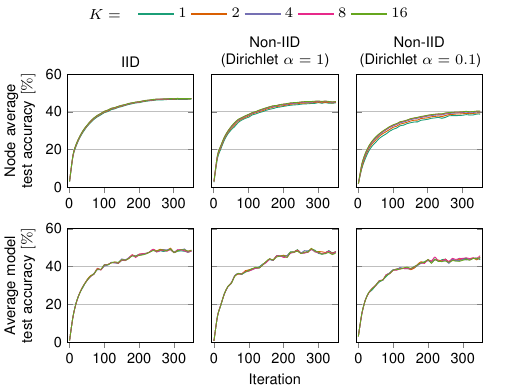}
    \caption{Node-average (top) and average-model (bottom) performance for \cifarhundred across iterations, data distributions and values of $K$.}
    \label{fig:plot5b_distribution_perf}
\end{figure}

\begin{figure}[h]
    \centering
    \includegraphics{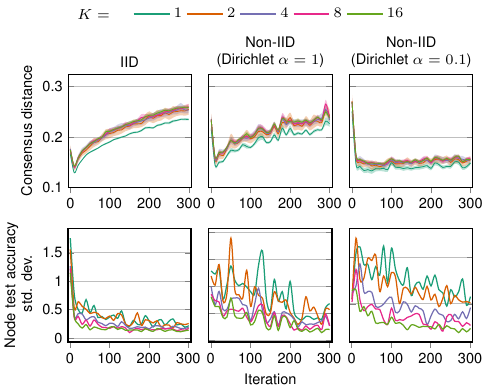}
    \caption{Consensus distance (top) and standard deviation of node-average performance (bottom) on \cifar across iterations, data distributions and values of $K$.}
    \label{fig:plot6_distribution_consensus}
\end{figure}

Finally, we explore the effect of the heterogeneity in data distribution on the impact of model fragmentation. \Cref{fig:plot5_distribution_perf,fig:plot5b_distribution_perf} both show three sets of plots, each showing node-average (top) and average-model (bottom) accuracy values. The first set shows an IID setting, the second a mildly non-IID setting ($\alpha=1$), and the third a strongly non-IID setting ($\alpha=0.1$). The Figure shows that the impact of fragmentation increases with the non-IID-ness of the data distribution. If data is IID, fragmentation brings no or little benefit, but as the non-IID-ness increases, the positive impact of fragmentation on node-average test accuracy increases. Average-model accuracy, however, remains independent of $K$ in all three settings. Finally, we observe that both metrics decrease their absolute values when moving to a more non-IID setting. Fragmentation thus provides the most benefit in this more difficult-to-train scenario.

\Cref{fig:plot6_distribution_consensus} provides further insights on the impact of non-IID-ness by showing consensus distance and node-accuracy standard deviation on \cifar. In all three settings, consensus distance increases with $K$ due to the non-convex setting, but it is interesting to see that it also increases with the number of iterations in all but the most non-IID setting. The standard deviation of node accuracy tends to decrease with $K$, confirming the trends observed above. However, differences across values of $K$ are more pronounced in more non-IID settings, which is precisely where fragmentation provides the greatest benefit. This is matched in~\cref{fig:plot8} in the Appendix on the \cifarhundred dataset.

\section{Related Work}
\label{sec:related}

\paragraph{\Acf{FL}} \Ac{FL} enables clients to collaboratively train a model by computing and sharing gradients locally~\cite{mcmahan2017communication, kairouz2021advances}. While \ac{FL} provides convergence rates on par with centralized settings~\cite{zinkevichParallelizedStochasticGradient2010,stichLocalSGDConverges2018}, it typically suffers from communication bottlenecks as every communication goes through the server. \ac{DL} overcomes this limitation by decentralizing training, but is in turn more affected by data heterogeneity~\cite{liFederatedOptimizationHeterogeneous2020,hsiehskewscout2020,belletDCliquesCompensatingData2022}.

\paragraph{Sparsification.} Sparsification techniques were originally designed as coordinate descent methods to reduce the cost of gradient descent in centralized training~\cite{nesterovEfficiencyCoordinateDescent2012, shamirStochasticGradientDescent2013}.
They were then adapted to reduce the communication overhead of \ac{FL} and \ac{DL}, by sharing only a subset of parameters instead of the entire model~\cite{alistarh2018sparseconvergence,tangCommunicationEfficientDecentralizedLearning2020}. By keeping track of the unshared parameters, it is possible to achieve state-of-the-art theoretical convergence rates while significantly reducing communication~\cite{stichSparsifiedSGDMemory2018}. Literature thus focuses on sharing the the components that yield the best tradeoff between convergence speed and communication~\cite{koloskovaDecentralizedStochasticOptimization2019,dhasade2023get,liu2023yoga}.

\paragraph{Fragmentation-based approaches.} Fragmentation-based approaches like Split Learning decentralize a model~\cite{thapa2022splitfed}, but do not consider scenarios where nodes hold an inference-capable model. \textsc{Shatter} and \textsc{DivShare}~\cite{biswas2024noiseless,biswas2025boosting} consider scenarios closer to ours, but either to address privacy concerns or to deal with stragglers in asynchronous~\ac{DL}, observing utility improvements from fragmentation. 
\section{Conclusion}
\label{sec:conclusion}

This work presents \sys, a novel \ac{DL} framework that studies the latent potential of model fragmentation. \sys provides a unifying paradigm-shifting approach to \ac{DL} that leverages model fragmentation to drive faster, more diverse information flow. By deriving  the theoretical convergence guarantees of \sys, we are the first to capture the evolution of the entire global model within a model-fragmented network architecture, showing that \sys maintains a competitive edge over the state-of-the-art \ac{EL} in terms of utility. We demonstrate that, for convex loss landscapes, \sys gains its edge by maximizing information diversity, effectively curbing redundant parameter sharing and ensuring that nodes receive uncorrelated updates and, therefore, optimizes the utility gain under the same communication budget. Extensive experiments demonstrate that \sys outperforms \ac{EL} in the presence of highly heterogenous node data, while remaining as good as \ac{EL} in more uniform settings.
In summary, \sys establishes a new, theoretically-grounded standard for efficient and scalable \ac{DL}.

\section*{Impact statement}
This paper presents work whose goal is to advance the field of machine learning. There are many potential societal consequences of our work, none of which we feel must be specifically highlighted here. 
\section*{Acknowledgments}
This work benefitted from the financial support of the international mobility grants program operated by the College Doctoral de Bretagne and co-funded by the Region de Bretagne and Rennes Métropole.
This work has been co-funded by the Swiss National Science Foundation, under the project ``FRIDAY: Frugal, Privacy-Aware and Practical Decentralized Learning'', SNSF proposal No. 10.001.796. 
\bibliography{main.bib}

@article{lian2017can,
  title={Can decentralized algorithms outperform centralized algorithms? a case study for decentralized parallel stochastic gradient descent},
  author={Lian, Xiangru and Zhang, Ce and Zhang, Huan and Hsieh, Cho-Jui and Zhang, Wei and Liu, Ji},
  journal={Advances in Neural Information Processing Systems},
  volume={30},
  year={2017},
  url={https://proceedings.neurips.cc/paper_files/paper/2017/file/f75526659f31040afeb61cb7133e4e6d-Paper.pdf}
}

@inproceedings{devos2023epidemic,
      title={Epidemic Learning: Boosting Decentralized Learning with Randomized Communication}, 
      author={Martijn de Vos and Sadegh Farhadkhani and Rachid Guerraoui and Anne-Marie Kermarrec and Rafael Pires and Rishi Sharma},
      year={2023},
      eprint={2310.01972},
      archivePrefix={arXiv},
      booktitle={37th Annual Conference on Neural Information Processing Systems (NeurIPS '23)}
}

@INPROCEEDINGS{dhasade2023get,
  author={Dhasade, Akash and Kermarrec, Anne-Marie and Pires, Rafael and Sharma, Rishi and Vujasinovic, Milos and Wigger, Jeffrey},
  booktitle={2023 IEEE 43rd International Conference on Distributed Computing Systems (ICDCS)}, 
  title={Get More for Less in Decentralized Learning Systems}, 
  year={2023},
  volume={},
  number={},
  pages={463-474},
  doi={10.1109/ICDCS57875.2023.00067}}

@InProceedings{pmlr-v119-koloskova20a,
  title = 	 {A Unified Theory of Decentralized {SGD} with Changing Topology and Local Updates},
  author =       {Koloskova, Anastasia and Loizou, Nicolas and Boreiri, Sadra and Jaggi, Martin and Stich, Sebastian},
  booktitle = 	 {Proceedings of the 37th International Conference on Machine Learning},
  pages = 	 {5381--5393},
  year = 	 {2020},
  editor = 	 {III, Hal Daumé and Singh, Aarti},
  volume = 	 {119},
  series = 	 {Proceedings of Machine Learning Research},
  month = 	 {13--18 Jul},
  publisher =    {PMLR},
  url = 	 {https://proceedings.mlr.press/v119/koloskova20a.html}
}

@inproceedings{mcmahan2017communication,
  title={Communication-efficient learning of deep networks from decentralized data},
  author={McMahan, Brendan and Moore, Eider and Ramage, Daniel and Hampson, Seth and y Arcas, Blaise Aguera},
  booktitle={Artificial intelligence and statistics},
  pages={1273--1282},
  year={2017},
  organization={PMLR}
}

@misc{leaf,
      title={LEAF: A Benchmark for Federated Settings}, 
      author={Sebastian Caldas and Sai Meher Karthik Duddu and Peter Wu and Tian Li and Jakub Konečný and H. Brendan McMahan and Virginia Smith and Ameet Talwalkar},
      year={2019},
      eprint={1812.01097},
      archivePrefix={arXiv},
      primaryClass={cs.LG}
}

@misc{grouplens:2021:movielens,
	year=2021,
	author={Grouplens},
	title={MovieLens Datasets},
	url={https://grouplens.org/datasets/movielens/},
}

@article{korenmatrixfactorization2009,
author = {Koren, Yehuda and Bell, Robert and Volinsky, Chris},
title = {Matrix Factorization Techniques for Recommender Systems},
year = {2009},
issue_date = {August 2009},
publisher = {IEEE Computer Society Press},
address = {Washington, DC, USA},
volume = {42},
number = {8},
issn = {0018-9162},
url = {https://doi.org/10.1109/MC.2009.263},
doi = {10.1109/MC.2009.263},
journal = {Computer},
month = {aug},
pages = {30--37},
numpages = {8}
}

@article{krizhevsky2014cifar,
  title={The CIFAR-10 dataset},
  author={Krizhevsky, Alex and Nair, Vinod and Hinton, Geoffrey},
  url={https://www.cs.toronto.edu/~kriz/cifar.html},
  volume={55},
  number={5},
  year={2014}
}

@inproceedings{alistarh2018sparseconvergence,
    author = {Alistarh, Dan and Hoefler, Torsten and Johansson, Mikael and Khirirat, Sarit and Konstantinov, Nikola and Renggli, C\'{e}dric},
    title = {The Convergence of Sparsified Gradient Methods},
    url={https://proceedings.neurips.cc/paper_files/paper/2018/file/314450613369e0ee72d0da7f6fee773c-Paper.pdf},
    year = {2018},
    location = {Montr\'{e}al, Canada},
    booktitle = {NeurIPS}
}

@inproceedings{koloskova2020decentralized,
  title={Decentralized Deep Learning with Arbitrary Communication Compression},
  author={Koloskova, Anastasia and Lin, Tao and Stich, Sebastian U and Jaggi, Martin},
  booktitle={ICLR},
  url={https://openreview.net/pdf?id=SkgGCkrKvH},
  year={2020}
}

@inproceedings{hsiehskewscout2020,
    author = {Hsieh, Kevin and Phanishayee, Amar and Mutlu, Onur and Gibbons, Phillip B.},
    title = {The Non-{IID} Data Quagmire of Decentralized Machine Learning},
    year = {2020},
    url={http://proceedings.mlr.press/v119/hsieh20a/hsieh20a.pdf},
    booktitle = {ICML}
}

@inproceedings{stichSparsifiedSGDMemory2018,
  title = {Sparsified {{SGD}} with {{Memory}}},
  booktitle = {Advances in {{Neural Information Processing Systems}}},
  author = {Stich, Sebastian U and Cordonnier, Jean-Baptiste and Jaggi, Martin},
  year = {2018},
  volume = {31},
  urldate = {2025-05-12},
  url = {https://proceedings.neurips.cc/paper/2018/hash/b440509a0106086a67bc2ea9df0a1dab-Abstract.html}
}

@article{nesterovEfficiencyCoordinateDescent2012,
  title = {Efficiency of {{Coordinate Descent Methods}} on {{Huge-Scale Optimization Problems}}},
  author = {Nesterov, {\relax Yu}.},
  year = {2012},
  month = jan,
  journal = {SIAM Journal on Optimization},
  volume = {22},
  number = {2},
  pages = {341--362},
  issn = {1052-6234, 1095-7189},
  doi = {10.1137/100802001},
  urldate = {2025-06-17},
  langid = {english},
}

@article{hendersonVecpermutationMatrixVec1981,
  title = {The Vec-Permutation Matrix, the Vec Operator and {{Kronecker}} Products: A Review},
  shorttitle = {The Vec-Permutation Matrix, the Vec Operator and {{Kronecker}} Products},
  author = {Henderson, Harold V. and Searle, S. R.},
  year = 1981,
  month = jan,
  journal = {Linear and Multilinear Algebra},
  volume = {9},
  number = {4},
  pages = {271--288},
  publisher = {Taylor \& Francis},
  issn = {0308-1087},
  doi = {10.1080/03081088108817379},
  urldate = {2025-09-15},
}

@article{loanUbiquitousKroneckerProduct2000,
  title = {The Ubiquitous {{Kronecker}} Product},
  author = {Loan, Charles F. Van},
  year = 2000,
  month = nov,
  journal = {Journal of Computational and Applied Mathematics},
  series = {Numerical {{Analysis}} 2000. {{Vol}}. {{III}}: {{Linear Algebra}}},
  volume = {123},
  number = {1},
  pages = {85--100},
  issn = {0377-0427},
  doi = {10.1016/S0377-0427(00)00393-9},
  urldate = {2025-03-19},
}

@article{langvilleKroneckerProductStochastic2004,
  title = {The {{Kronecker}} Product and Stochastic Automata Networks},
  author = {Langville, Amy N. and Stewart, William J.},
  year = 2004,
  month = jun,
  journal = {Journal of Computational and Applied Mathematics},
  volume = {167},
  number = {2},
  pages = {429--447},
  issn = {0377-0427},
  doi = {10.1016/j.cam.2003.10.010},
  urldate = {2025-03-19},
}

@misc{zhuSurprisingEffectivenessSingle2025,
  title = {On the {{Surprising Effectiveness}} of a {{Single Global Merging}} in {{Decentralized Learning}}},
  author = {Zhu, Tongtian and Zhang, Tianyu and Wang, Mingze and Zhou, Zhanpeng and Wang, Can},
  year = 2025,
  month = oct,
  number = {arXiv:2507.06542},
  eprint = {2507.06542},
  primaryclass = {cs},
  publisher = {arXiv},
  doi = {10.48550/arXiv.2507.06542},
  urldate = {2026-01-25},
  archiveprefix = {arXiv},
  langid = {english},
}

@inproceedings{kongConsensusControlDecentralized2021,
  title = {Consensus {{Control}} for {{Decentralized Deep Learning}}},
  booktitle = {Proceedings of the 38th {{International Conference}} on {{Machine Learning}}},
  author = {Kong, Lingjing and Lin, Tao and Koloskova, Anastasia and Jaggi, Martin and Stich, Sebastian},
  year = 2021,
  month = jul,
  pages = {5686--5696},
  publisher = {PMLR},
  issn = {2640-3498},
  urldate = {2026-01-26},
  langid = {english},
}

@article{biswas2024noiseless,
   title={Noiseless Privacy-Preserving Decentralized Learning},
   volume={2025},
   ISSN={2299-0984},
   url={http://dx.doi.org/10.56553/popets-2025-0043},
   DOI={10.56553/popets-2025-0043},
   number={1},
   journal={Proceedings on Privacy Enhancing Technologies},
   publisher={Privacy Enhancing Technologies Symposium Advisory Board},
   author={Biswas, Sayan and Even, Mathieu and Kermarrec, Anne-Marie and Massoulié, Laurent and Pires, Rafael and Sharma, Rishi and de Vos, Martijn},
   year={2025},
   month=jan, pages={824–844} }

@inproceedings{biswas2025boosting,
  title={Boosting asynchronous decentralized learning with model fragmentation},
  author={Biswas, Sayan and Kermarrec, Anne-Marie and Marouani, Alexis and Pires, Rafael and Sharma, Rishi and De Vos, Martijn},
  booktitle={Proceedings of the ACM on Web Conference 2025},
  pages={685--696},
  year={2025}
}

@inproceedings{shamirStochasticGradientDescent2013,
  title = {Stochastic {{Gradient Descent}} for {{Non-smooth Optimization}}: {{Convergence Results}} and {{Optimal Averaging Schemes}}},
  shorttitle = {Stochastic {{Gradient Descent}} for {{Non-smooth Optimization}}},
  booktitle = {Proceedings of the 30th {{International Conference}} on {{Machine Learning}}},
  author = {Shamir, Ohad and Zhang, Tong},
  year = 2013,
  month = feb,
  pages = {71--79},
  publisher = {PMLR},
  issn = {1938-7228},
  urldate = {2025-06-17},
  langid = {english},
}

@article{kairouz2021advances,
  title={Advances and open problems in federated learning},
  author={Kairouz, Peter and McMahan, H Brendan and Avent, Brendan and Bellet, Aur{\'e}lien and Bennis, Mehdi and Bhagoji, Arjun Nitin and Bonawitz, Kallista and Charles, Zachary and Cormode, Graham and Cummings, Rachel and others},
  journal={Foundations and trends{\textregistered} in machine learning},
  volume={14},
  number={1--2},
  pages={1--210},
  year={2021},
  publisher={Now Publishers, Inc.}
}

@inproceedings{belletDCliquesCompensatingData2022,
  title = {D-{{Cliques}}: {{Compensating}} for {{Data Heterogeneity}} with {{Topology}} in {{Decentralized Federated Learning}}},
  shorttitle = {D-{{Cliques}}},
  booktitle = {2022 41st {{International Symposium}} on {{Reliable Distributed Systems}} ({{SRDS}})},
  author = {Bellet, Aur{\'e}lien and Kermarrec, Anne-Marie and Lavoie, Erick},
  year = 2022,
  month = sep,
  pages = {1--11},
  issn = {2575-8462},
  doi = {10.1109/SRDS55811.2022.00011},
  urldate = {2026-01-28},
}

@inproceedings{koloskovaDecentralizedStochasticOptimization2019,
  title = {Decentralized {{Stochastic Optimization}} and {{Gossip Algorithms}} with {{Compressed Communication}}},
  booktitle = {Proceedings of the 36th {{International Conference}} on {{Machine Learning}}},
  author = {Koloskova, Anastasia and Stich, Sebastian and Jaggi, Martin},
  year = 2019,
  month = may,
  pages = {3478--3487},
  publisher = {PMLR},
  issn = {2640-3498},
  urldate = {2026-01-28},
}

@inproceedings{tangCommunicationEfficientDecentralizedLearning2020,
  title = {Communication-{{Efficient Decentralized Learning}} with {{Sparsification}} and {{Adaptive Peer Selection}}},
  booktitle = {2020 {{IEEE}} 40th {{International Conference}} on {{Distributed Computing Systems}} ({{ICDCS}})},
  author = {Tang, Zhenheng and Shi, Shaohuai and Chu, Xiaowen},
  year = 2020,
  month = nov,
  pages = {1207--1208},
  issn = {2575-8411},
  doi = {10.1109/ICDCS47774.2020.00153},
  urldate = {2026-01-28}
}

@article{liFederatedOptimizationHeterogeneous2020,
  title = {Federated {{Optimization}} in {{Heterogeneous Networks}}},
  author = {Li, Tian and Sahu, Anit Kumar and Zaheer, Manzil and Sanjabi, Maziar and Talwalkar, Ameet and Smith, Virginia},
  year = 2020,
  month = mar,
  journal = {Proceedings of Machine Learning and Systems},
  volume = {2},
  pages = {429--450},
  urldate = {2026-01-28},
  langid = {english},
}

@inproceedings{stichLocalSGDConverges2018,
  title = {Local {{SGD Converges Fast}} and {{Communicates Little}}},
  booktitle = {International {{Conference}} on {{Learning Representations}}},
  author = {Stich, Sebastian U.},
  year = 2018,
  month = sep,
  urldate = {2026-01-28},
  langid = {english},
}

@inproceedings{zinkevichParallelizedStochasticGradient2010,
  title = {Parallelized {{Stochastic Gradient Descent}}},
  booktitle = {Advances in {{Neural Information Processing Systems}}},
  author = {Zinkevich, Martin and Weimer, Markus and Li, Lihong and Smola, Alex},
  year = 2010,
  volume = {23},
  publisher = {Curran Associates, Inc.},
  urldate = {2026-01-28}
}

@article{liu2023yoga,
  title={Yoga: Adaptive layer-wise model aggregation for decentralized federated learning},
  author={Liu, Jun and Liu, Jianchun and Xu, Hongli and Liao, Yunming and Wang, Zhiyuan and Ma, Qianpiao},
  journal={IEEE/ACM Transactions on Networking},
  volume={32},
  number={2},
  pages={1768--1780},
  year={2023},
  publisher={IEEE}
}

@inproceedings{thapa2022splitfed,
  title={Splitfed: When federated learning meets split learning},
  author={Thapa, Chandra and Arachchige, Pathum Chamikara Mahawaga and Camtepe, Seyit and Sun, Lichao},
  booktitle={Proceedings of the AAAI conference on artificial intelligence},
  volume={36},
  number={8},
  pages={8485--8493},
  year={2022}
}
\bibliographystyle{icml2026}

\clearpage
\section{Additional Experiments}
\label{app:additional_experiments}

The experiments with \cifarhundred and \shakespeare datasets are presented in \Cref{fig:plot1a,fig:plot2a,fig:plot8}.

\begin{figure}[t]
    \centering
    \includegraphics{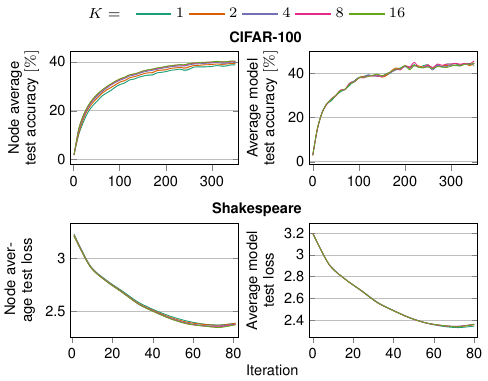}
    \caption{Performance of \sys across iterations and number of fragments ($K$) for \cifarhundred (top row) and \shakespeare (bottom row), showing node-average (left column) and model-average (right column) test accuracies.}
    \label{fig:plot1a}
\end{figure}

\begin{figure}[t]
    \centering
    \includegraphics{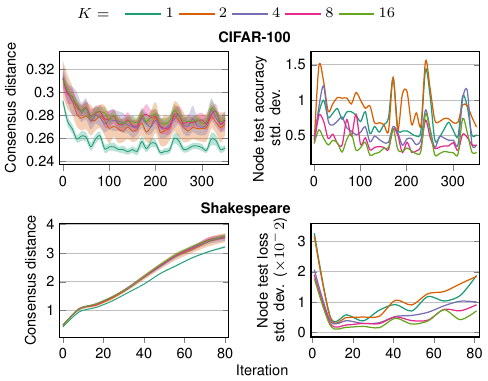}
    \caption{Consensus distance (left) and standard deviation of node performance (right) across iterations and $K$, for \cifarhundred (top) and \shakespeare (bottom), on a network with degree 8.}
    \label{fig:plot2a}
\end{figure}

\begin{figure}[h]
	\centering
    \includegraphics{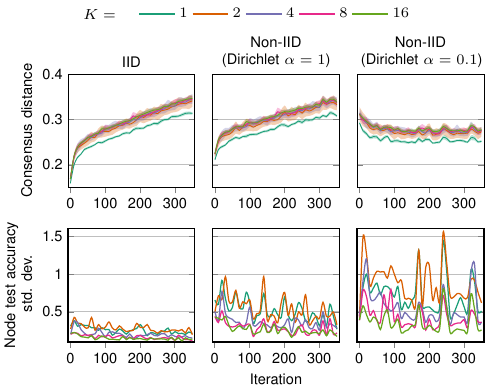}
	\caption{Consensus distance (top) and standard deviation of node-average performance (bottom) on \cifarhundred across iterations, data distributions and values of $K$.}
	\label{fig:plot8}
\end{figure}

\section{Assumptions}
\label{sec:assumptions}
We now proceed to detail the necessary assumptions to prove~\cref{thm:convergence of sys}. All the assumptions here are standard in the \ac{DL} literature~\cite{lian2017can,devos2023epidemic,pmlr-v119-koloskova20a}, and are not specific to our approach.

\begin{assumption}[$\smoothnessconstant$-smoothness]\label{ass:smoothness}
    We assume that the local loss functions are smooth for some constant $\smoothnessconstant$:
    for any $x,y\in\R^\nbparams$, it holds that
    \begin{align}
        \norm{\localgradient{i}\left(x\right) - \localgradient{i}\left(y\right)}
        \leq
        \smoothnessconstant \norm{x - y}
        \label{eq:smoothness}
    \end{align}
\end{assumption}

\begin{assumption}[Bounded stochastic noise]\label{ass:bounded stochastic noise}
    For a fixed $x\in\modelspace$:
    \begin{align}\label{eq:bounded stochastic noise}
        \esp{\norm{\nabla f_i\left(x,\datapoint{i}{t}\right) - \localgradient{i}\left(x\right)}^2} \leq \boundstochasticnoise^2
    \end{align}
\end{assumption}

We also assume a bound $\boundheterogeneity$ on the local heterogeneity between nodes.
\begin{assumption}[Bounded heterogeneity]\label{ass:bounded heterogeneity} There exist some $\boundheterogeneity\in\R{}$ such that for any $x \in \modelspace$:
    \begin{align}\label{eq:bounded heterogeneity}
        \frac{1}{n}\sum_{i=1}^{n} \norm{\localgradient{i}\left(x\right) - \globalgradient\left(x\right)}^2 \leq \boundheterogeneity^2
    \end{align}
    where $\globalgradient\left(x\right) = \frac{1}{n}\sum_{i=1}^{n} \localgradient{i}\left(x\right)$.
\end{assumption}

\section{Useful Lemmas}\label{sec:useful lemmas}
In the proof, we will use the following common properties.
\begin{lemma}
    For any $a,b\in\R^\nbparams$ and any $\alpha>0$:
    \begin{align}\label{eq:splitting norm of sum arbitrary constant}
        \norm{a+b}^2\leq \left(1+\alpha\right)\norm{a}^2 
        + \left(1+\alpha^{-1}\right)\norm{b}^2.
    \end{align}
    
\end{lemma}

\begin{lemma}\label{lem:dot product bound}
    Consider two real series $\left(a_i\right)_{i\in\integerinterval{1,n}},\left(b_i\right)_{i\in\integerinterval{1,n}}$ such that $a_i\geq 0, b_i\geq 0$.
    Then:
    \begin{align}\label{eq:dot product bound}
        \sum_{i=1}^{n}a_i b_i \leq \left(\sum_{i=1}^{n}a_i\right)\left(\sum_{i=1}^{n}b_i\right)
    \end{align}
\end{lemma}

\begin{lemma}\label{lem:variance norm rewording}
    \begin{align}\label{eq:variance norm rewording}
        \frac{1}{n}\sum_{i=1}^{n}\norm{x_i-\bar{x}}^2 = \frac{1}{2n^2}\sum_{i=1}^{n}\sum_{j=1}^{n}\norm{x_i-x_j}^2
    \end{align}
\end{lemma}

\begin{lemma}\label{lem:norm of sum} For any set of vector $\left(\model{i}{}\right)$, it holds that:
    \begin{align}
        \norm{\sum_{i=1}^{n}\model{i}{}}^2 \leq& n \sum_{i=1}^{n}\norm{\model{i}{}}^2 
    \end{align}
    
\end{lemma} 
\section{Missing proofs}\label{sec:proofs}
We now restate and provide the missing proofs for theorems from \Cref{sec:analysis}.

\subsection{Proof of \cref{thm:convergence of sys}}

To prove~\cref{thm:convergence of sys}, we will consider~\cref{ass:bounded heterogeneity,ass:bounded stochastic noise,ass:smoothness} to match assumptions from~\cite{devos2023epidemic}.

\subsubsection{Proof of intermediary lemmas}
The proof of this theorem will strongly match the original work's. We first restate all lemmas whose proof is unchanged:

\begin{lemma}\label{lem:consensus_and_grad_disagreement}
Suppose that~\cref{ass:bounded heterogeneity,ass:bounded stochastic noise,ass:smoothness} hold. Consider~\cref{alg:system} and let $\gamma$ be a step-size with $\gamma \le \tfrac{1}{20L}$. Then for any $t\ge 0$ we have
\[
\frac{1}{n^2}\sum_{i,j\in[n]}\E\big\|x_t^{(i)}-x_t^{(j)}\big\|^2 \le 20\frac{1+3\eta_s}{(1-\eta_s)^2}\,\beta_s\,\gamma^2\big(\sigma^2+\mathcal{H}^2\big),
\]
and
\[
\frac{1}{n^2}\sum_{i,j\in[n]}\E\big\|g_t^{(i)}-g_t^{(j)}\big\|^2 \le 15\big(\sigma^2+\mathcal{H}^2\big),
\]
with $\beta_s:= \frac{1}{s}\left(1-\left(1-\frac{s}{n-1}\right)^n\right)-\frac{1}{n-1}$ 
\end{lemma}

\begin{lemma}\label{lem:gradient_norm_bound}
Suppose that~\cref{ass:bounded heterogeneity,ass:bounded stochastic noise} hold. Consider~\cref{alg:system} with $\lr{} \le \tfrac{1}{2L}$. For any $t\in\{0,\dots,T-1\}$ it holds that
\begin{align*}
\E\big\|\nabla F(\bar x_t)\big\|^2 
&\le \frac{2}{\gamma}\,\E\big[F(\bar x_t)-F(\bar x_{t+1})\big] 
\\&+ \frac{L^2}{2n^2}\sum_{i,j\in[n]}\E\big\|x_t^{(i)}-x_t^{(j)}\big\|^2 
+ \frac{2L\gamma\sigma^2}{n} 
\\&+ \frac{2L}{\gamma}\,\E\big\|\bar x_{t+1}-\bar x_{t+1/2}\big\|^2.
\end{align*}
\end{lemma}

\begin{lemma}[Mixing lemma]\label{lem:mixing}
Consider~\cref{alg:system}. Let $n\ge 2$, $s\ge 1$, $T\ge 1$, and $t\in\{0,\dots,T-1\}$. For the EL-Local interaction model (each node share their model with $s$ other nodes), we have:
\begin{enumerate}[label=(\alph*)]
    \item $\E[\bar x_{t+1}] = \E[\bar x_{t+1/2}]$,
    \item $\frac{1}{n^2}\sum_{i,j\in[n]}\E\norm{x_{t+1}^{(i)}-x_{t+1}^{(j)}}^2 \le \beta_s \,\frac{1}{n^2}\sum_{i,j\in[n]}\E\norm{x_{t+1/2}^{(i)}-x_{t+1/2}^{(j)}}^2,$
    \item $\E\norm{\bar x_{t+1}-\bar x_{t+1/2}}^2 \le \frac{\beta_s}{2n}\sum_{i,j\in[n]}\E \norm{x_{t+1/2}^{(i)}-x_{t+1/2}^{(j)}}^2,$
\end{enumerate}
\end{lemma}
\begin{proof}
For $k=1,\dots,K$ remember that
\(\Pi^{(k)}:\R^d\to\R^d\) is the orthogonal projector that selects the coordinates in the $k$-th fragment. The projectors satisfy
\[
\Pi^{(k)}\Pi^{(q)}=0 \ (k\ne q),\qquad \sum_{k=1}^K \Pi^{(k)} = I_d.
\]

We will consider norms over fragments subspaces $\norm{.}_{\Pi^{(k)}}$ defined by the projectors, and apply existing results on those subspaces. Since $\sum_{k=1}^{K} \Pi^{(k)} = I_d$, any vector $x\in\R^d$ decomposes as
\[x = \sum_{k=1}^K \Pi^{(k)} x, \quad \text{and} \quad \norm{x}^2 = \sum_{k=1}^K \norm{x}_{\Pi^{(k)}}^2.\]

\textbf{(a)} Each per-fragment averaging step uses the matrix $W_t^{(k)}$ to produce
\(\Pi^{(k)} x_{t+1}^{(i)} = \sum_j W_t^{(k)}[i,j] \,\Pi^{(k)} x_{t+1/2}^{(j)}\).
Taking the network mean over $i$ and using linearity gives
\[
\bar x_{t+1} = \frac{1}{n}\sum_{i=1}^n x_{t+1}^{(i)} = \frac{1}{n}\sum_{j=1}^n \Big(\sum_{i=1}^n W_t^{(k)}[i,j]\Big)\,\Pi^{(k)} x_{t+1/2}^{(j)}
\]
for each fragment; because the randomized EL-Local interaction preserves the coordinate sums in expectation (see Lemma~1 of~\cite{devos2023epidemic} and its discussion), and by linearity of the expectation, we obtain
\(\E[\bar x_{t+1}] = \E[\bar x_{t+1/2}]\).

\textbf{(b)} Since the fragments partition the coordinates, the pairwise disagreement decomposes over fragments:
\begin{align*}
\frac{1}{n^2}&\sum_{i,j=1}^n \E\|x_{t+1}^{(i)}-x_{t+1}^{(j)}\|^2
\\&= \frac{1}{n^2}\sum_{i,j}\E\norm{\sum_{k=1}^K \Pi^{(k)}\big(x_{t+1}^{(i)}-x_{t+1}^{(j)}\big)}^2
\\&= \frac{1}{n^2}\sum_{i,j}\E\norm{\sum_{k=1}^K \big(x_{t+1}^{(i)}-x_{t+1}^{(j)}\big)}_{\Pi^{(k)}}^2.
\end{align*}

Fix a fragment index $k$. Restricting the dynamics to the coordinates selected by $\Pi^{(k)}$, the gossiping is performed with $W_t^{(k)}$. Applying Lemma~1 of~\cite{devos2023epidemic} (EL-Local case) to those subspaces yields:
\begin{align*}
    \frac{1}{n^2}&\sum_{i,j}\E\norm{x_{t+1}^{(i)}-x_{t+1}^{(j)}}_{\Pi^{(k)}}^2 
    \\ &\le \beta_s\,\frac{1}{n^2}\sum_{i,j}\E\norm{x_{t+1/2}^{(i)}-x_{t+1/2}^{(j)}}_{\Pi^{(k)}}^2.
\end{align*}
Summing this inequality for $k=1,\dots,K$ yields statement (b).

\textbf{(c)} We have the decomposition:
\[
\bar x_{t+1}-\bar x_{t+1/2} = \sum_{k=1}^K \frac{1}{n}\sum_{i=1}^n \big(\Pi^{(k)}x_{t+1}^{(i)} - \Pi^{(k)}x_{t+1/2}^{(i)}\big).
\]

Thus, we can write:
\begin{align*}
\E&\norm{\bar x_{t+1}-\bar x_{t+1/2}}^2 
\\&= \E\norm{\sum_{k=1}^K \frac{1}{n}\sum_{i=1}^n \big(\Pi^{(k)}x_{t+1}^{(i)} - \Pi^{(k)}x_{t+1/2}^{(i)}\big)}^2
\\&= \sum_{k'=1}^{K}\E\norm{\sum_{k=1}^K\Pi^{(k)} \frac{1}{n}\sum_{i=1}^n \big(x_{t+1}^{(i)} - x_{t+1/2}^{(i)}\big)}_{\Pi^{(k')}}^2
\\&= \sum_{k=1}^{K}\E\norm{\frac{1}{n}\sum_{i=1}^n \big(x_{t+1}^{(i)} - x_{t+1/2}^{(i)}\big)}_{\Pi^{(k)}}^2,
\end{align*}
where we used the fact that different fragments are orthogonal.
Applying Lemma~1 of~\cite{devos2023epidemic} (EL-Local case) to each $\sum_{k=1}^{K}\E\norm{\frac{1}{n}\sum_{i=1}^n \big(x_{t+1}^{(i)} - x_{t+1/2}^{(i)}\big)}_{\Pi^{(k)}}^2$, we obtain:
\begin{align*}
    \E&\norm{\bar x_{t+1}-\bar x_{t+1/2}}^2 
    \\&\le \sum_{k=1}^{K} \frac{\beta_s}{2n^3}\sum_{i,j=1}^n \E\norm{x_{t+1/2}^{(i)} - x_{t+1/2}^{(j)}}_{\Pi^{(k)}}^2
    \\&= \frac{\beta_s}{2n^3}\sum_{i,j=1}^n \E\norm{x_{t+1/2}^{(i)} - x_{t+1/2}^{(j)}}^2,
\end{align*}
which concludes the proof.
\end{proof}

Now that all intermediary lemmas are stated, we can proceed to the main proof.
\convergenceEpidemic*

\begin{proof}
The proof of the main theorem in \cite{devos2023epidemic} does not rely on the gossip parameters, but only on the intermediary lemmas stated above. Since we have shown those lemmas hold in our setting, we can perform the exact same proof steps to derive the desired result. We refer the reader to \cite{devos2023epidemic} for the full proof details.
\end{proof}

\subsection{Proof of \cref{thm:consensus error evolution}}
We restate and prove \Cref{thm:consensus error evolution} below. Recall that we use the notation from \Cref{subsec:theory impact of chunks}.

\vectorizedconvergence*

\begin{proof}[Proof of \Cref{thm:consensus error evolution}]
        Define the projector for consensus error:
        \[P := \left(\identity{n} - \tfrac{1}{n}\one_n\one_n^\top\right)\kronecker I_d,
        \]
        so that the consensus error is written as $e_t := PX_t = X_t-\bar{X}_t$. Note that
        $PX^*=0$ because $X^*=\one_n\kronecker x^*$ is a consensus vector.
        Using the update equations~\cref{eq:vectorized gossip update,eq:vectorized gradient update}, we have:
        \begin{align*}
                e_{t+1} &= P X_{t+1} \\
                &= P\,\commutingmatrix{n}{d}\mathbf{W}_t\commutingmatrix{d}{n}\Big(X_t-2\eta\,(I_n\kronecker A)(X_t-X^*)\Big).
        \end{align*}

        Decompose the error into disagreement and consensus parts:
        \[X_t-X^*=e_t+\big(\bar X_t-X^*\big),\qquad
            \bar X_t-X^*=\one_n\kronecker(\bar x_t-x^*),
        \]
        so the second summand is a consensus vector. By assumption each gossip block satisfies
        $W_t^{(k)}\,\one_n=\one_n$, hence
        $\commutingmatrix{n}{d}\mathbf{W}_t\commutingmatrix{d}{n}$ maps consensus vectors to consensus vectors. In particular
        \[(I_n\kronecker A)(\bar X_t-X^*)=\one_n\kronecker\big(A(\bar x_t-x^*)\big)
        \]
        is a consensus vector, and therefore
        \[P\,\commutingmatrix{n}{d}\mathbf{W}_t\commutingmatrix{d}{n}\,(I_n\kronecker A)(\bar X_t-X^*)=0.
        \]
        Using linearity and the decomposition above we obtain
        \begin{align*}
                e_{t+1}
                &= P\,\commutingmatrix{n}{d}\mathbf{W}_t\commutingmatrix{d}{n} e_t
                \\&-2\eta\,P\,\commutingmatrix{n}{d}\mathbf{W}_t\commutingmatrix{d}{n}(I_n\kronecker A) e_t,
                \\&= P\underbrace{\commutingmatrix{n}{d}\mathbf{W}_t\commutingmatrix{d}{n} \left(\identity{nd} -2\eta(I_n\kronecker A)\right)}_{:=M_t} e_t,
        \end{align*}
        since the term coming from $(\bar X_t-X^*)$ vanishes after left-multiplication by
        $P\,\commutingmatrix{n}{d}\mathbf{W}_t\commutingmatrix{d}{n}$.
        Finally, factorizing by $P\,\commutingmatrix{n}{d}\mathbf{W}_t\commutingmatrix{d}{n}$ yields:
        \begin{align*}
                e_{t+1}
                &= P\,\commutingmatrix{n}{d}\mathbf{W}_t\commutingmatrix{d}{n}\left(I_n\kronecker \left(I_{d}-2\eta A\right)\right)e_t.
        \end{align*}

\end{proof} 
\end{document}